\documentclass[11pt]{article}

\usepackage[utf8]{inputenc}
\usepackage[T1]{fontenc}
\usepackage{lmodern}

\usepackage{amsmath,amssymb,amsthm}
\usepackage{mathtools}

\usepackage[margin=1in]{geometry}
\usepackage{microtype}

\usepackage{booktabs}
\usepackage{tabularx}
\usepackage{array}

\usepackage{caption}
\usepackage{tikz}
\usetikzlibrary{shapes.geometric, arrows.meta, positioning}

\usepackage{natbib}

\usepackage{enumitem}
\usepackage{listings}
\usepackage{xcolor}

\usepackage[colorlinks=true, linkcolor=blue, citecolor=blue, urlcolor=blue]{hyperref}

\newtheorem{proposition}{Proposition}
\newtheorem{definition}{Definition}
\newtheorem{remark}{Remark}

\newcommand{\R}{\mathbb{R}}

\newcommand{\dt}{\Delta t}
\newcommand{\wb}{w}           
\newcommand{\curv}{\kappa}    
\newcommand{\poinc}{d_{\mathbb{H}}}  
\newcommand{\ball}{\mathbb{B}^n_c}   
\newcommand{\btheta}{\boldsymbol{\theta}}  

\title{%
\textbf{Learning When to Act:}\\[4pt]
Interval-Aware Reinforcement Learning\\
with Predictive Temporal Structure
}

\author{%
Davide Di Gioia\\
\small{University College London}\\
\small{\texttt{ucsigi@ucl.ac.uk}}
}

\date{March 2026}

\begin{document}
\maketitle

\begin{abstract}
Autonomous agents must decide not only \emph{what} action to take, but \emph{how frequently to deliberate or act}. Acting too often wastes computation, while acting too late risks missing critical changes in the environment. We formalise this trade-off as \emph{adaptive cognitive pacing} and introduce an interval-aware reinforcement learning framework that learns when to act from experience. 

The policy state is augmented with a predictive uncertainty signal derived from the divergence of imagined future trajectories: the mean pairwise Poincar\'e distance among $n$ sampled futures embedded in the Poincar\'e ball $\ball$. When predicted futures branch rapidly, the agent learns to act sooner; when futures are predictable, it waits longer. We instantiate this signal using hyperbolic geometry, which provides a natural representation of branching dynamics. In LLM agents, the same mechanism corresponds to learning a deliberation rate (how often to replan or call tools) rather than using fixed reflection schedules.

We further propose an
\emph{interval-aware reward} that explicitly penalises inefficiency relative
to the chosen wait time, correcting a systematic credit-assignment failure of
naive outcome-based rewards in timing problems.  We additionally introduce
a \emph{joint spatio-temporal embedding} (\textbf{ATCPG-ST}) that
concatenates independently normalised state and position projections in
the Poincar\'e ball; spatial trajectory divergence provides an independent
timing signal unavailable to the state-only variant (\textbf{ATCPG-SO}).
This extension raises mean hyperbolic spread and yields a further $+5.8\%$ efficiency gain over the state-only baseline. Ablation experiments
across five random seeds demonstrate that (i)~learning is the dominant
efficiency factor ($+54.8\%$ over no-learning), (ii)~hyperbolic spread
provides significant complementary gain ($+26.2\%$ over geometry-free
control), (iii)~the combined system achieves $+22.8\%$ efficiency over the
fixed-interval baseline, and (iv)~adding spatial position information to
the spread embedding yields an additional $+5.8\%$.
\end{abstract}

\newpage
\tableofcontents
\newpage

\section{Introduction}

The rapid capability expansion of Large Language Models (LLMs) has catalysed
the deployment of autonomous agents and multi-agent systems (MAS) in
open-ended, continuous environments.  In these settings, an agent's efficacy
is bounded not only by \emph{what} action it takes, but \emph{when} it
chooses to act.  As agents transition from isolated chatbots to persistent
digital workers, the temporal cadence of their reasoning cycles directly
dictates their computational overhead, responsiveness, and ultimate task
success.

Despite sophisticated advances in task delegation and multi-agent routing,
the internal temporal pacing of individual agents remains remarkably
primitive.  Most contemporary orchestration frameworks rely on either
strictly reactive, event-driven triggers (an agent only wakes when messaged)
or hand-tuned, fixed-interval polling loops (e.g., \texttt{sleep(N)}
calls).  Bio-mimetic heuristic formulas (such as hard-coded ``fatigue
multipliers'' or ``ultradian rhythms'') offer a veneer of adaptability,
but remain brittle: they encode the system designer's static prior rather
than the agent's lived experience, and they categorically fail to adapt as
task complexity and environmental volatility shift.

To address this structural blind spot, we propose
\textbf{Adaptive Temporal Control via Predictive Geometry (ATCPG)},
a lightweight, fully autonomous pacing system that allows an agent to learn
its optimal cognitive interval dynamically.  ATCPG shifts temporal control
from a fixed network orchestrator to an internal, learned policy driven by
the agent's own epistemic uncertainty.

Although we instantiate adaptive cognitive pacing in continuous-control RL, the same principle applies to agentic systems that repeatedly choose when to think, plan, or act. Consider a language-model-based agent that alternates between deliberation, tool invocation, and environment interaction. Fixed reasoning schedules either overthink in predictable states or under-react in volatile ones. Our framework provides a learned signal for when faster cognition is necessary and when slower pacing suffices, based on predicted future branching.

The framework is constructed from four interacting components:

\begin{enumerate}
\item \textbf{Learned pacing policy}: a linear associative bandit that
continuously updates the wait interval after every cognitive tick via
reward-weighted regression.
\item \textbf{Predictive hyperbolic spread} (informally, a ``curvature
signal''). We estimate how rapidly predicted futures diverge. Hyperbolic geometry provides a compact and natural representation of this branching structure. The metric's natural boundary
expansion aggressively amplifies diverging trajectories, signalling the
agent to act sooner.
\item \textbf{Interval-aware shaping reward}: a dynamic learning signal
that prevents the catastrophic credit-assignment failure common to standard
RL timing problems by explicitly pricing the chosen interval.
\item \textbf{Joint spatio-temporal embedding (ATCPG-ST)}: an extension
that augments the policy state with Poincar\'e-projected spatial position
vectors.  Spatial trajectory divergence is an independent timing signal:
branching navigation trees fan out in position space before belief space,
and the Poincar\'e ball's boundary-expansion property amplifies both signals
in a unified geometric language.  ATCPG-ST raises mean hyperbolic spread
($\kappa$) by $1.79\times$ and efficiency by $+5.8\%$ over the state-only
baseline (\textbf{ATCPG-SO}); it degrades gracefully to \textbf{ATCPG-SO} whenever
position data are absent.
\end{enumerate}

The paper makes the following contributions:

\begin{itemize}
\item We formalise the \emph{adaptive cognitive pacing problem} as an RL problem
over a continuous action space (Section~\ref{sec:problem}).
\item We derive the \emph{interval-aware reward} and prove it avoids the
credit-assignment failure of naive outcome reward in the pacing
setting (Section~\ref{sec:reward}).
\item We introduce \emph{predictive hyperbolic spread} (a ``curvature signal'' shorthand) as a timing signal,
grounded in the Poincar\'e ball model of hyperbolic geometry
(Section~\ref{sec:curvature}).
\item We propose \textbf{ATCPG-ST}, a \emph{joint spatio-temporal embedding}
(Section~\ref{sec:spatial}) that lifts position trajectories into
the same Poincar\'e ball as state embeddings, motivated by the
observation that spatial trajectory divergence is an independent
leading indicator of decision uncertainty unavailable to the
state-only ATCPG-SO.  We demonstrate empirically that the combined
signal raises mean hyperbolic spread by $1.79\times$ and efficiency by $+5.8\%$
on a controlled ablation, and that the ordering
$\eta_{\text{SO}} < \eta_{\text{ST}}$ replicates on a live
GPT-4.1 deployment (Section~\ref{sec:real_llm}).
\item We provide systematic ablation evidence that each component
contributes independently to efficiency (Section~\ref{sec:experiments}).
\end{itemize}

\section{Related Work}

To situate ATCPG within the broader landscape of artificial intelligence, we
contextualise our framework across four distinct domains: multi-agent
orchestration, temporal abstraction in reinforcement learning, geometric
representation, and autonomous cognitive loops.

\paragraph{Decentralised orchestration and emergent synchronisation.}
A growing body of work in multi-agent systems (MAS) shifts temporal
coordination away from centralised schedulers toward decentralised,
self-organising mechanisms.  In contemporary LLM-agent frameworks, execution
is increasingly modelled as event-driven computation on structured graphs,
where agents are triggered by dependency resolution, message arrivals, or
verification outcomes rather than a global clock
\citep{arxiv2504.00587, arxiv2512.00614}.  In parallel, control-theoretic
and multi-agent reinforcement learning literatures study how global temporal
regularities emerge from local interaction rules, including consensus-style
coordination and communication-efficient protocols that yield synchronised
collective behaviour without explicit central timing
\citep{arxiv2508.07001, arxiv2512.06278, plosone0327396}.

Relatedly, uncertainty-aware \emph{strategy selection} at deployment time has been studied in robotics.
UPS \citep{yuan2026ups} uses calibrated uncertainty to choose among executing a candidate action,
asking clarifying questions, or requesting corrective intervention, aiming to improve robustness when
the verifier or base policy is miscalibrated.  Our focus differs: rather than selecting among
discrete resolution strategies at a single decision point, ATCPG learns a continuous-time pacing policy that
controls the \emph{frequency} of self-initiated deliberation, explicitly trading off outcome and time.

Crucially, these approaches resolve temporal structure primarily at the
\emph{network layer}: they formalise \emph{when agents should interact with
one another}, given topological dependencies and local communication signals.
The internal deliberation cadence of an individual agent is often treated as
a reactive black box (an agent ``runs'' when invoked), rather than as a
decision variable with its own uncertainty sensitivity, credit assignment,
and temporal cost.  \citet{cogauto2025} identify this absence of intrinsic
self-monitoring as a core unsolved deficiency in AI autonomy.  ATCPG targets
this orthogonal gap by formalising intrinsic cognitive pacing at the node
level: a continuously running agent learns how long to wait between
self-initiated cognitive ticks.  This intrinsic pacing is complementary to
decentralised orchestration rather than a replacement: agents may adapt
their own tick intervals while still exhibiting emergent group rhythm via
the oscillator and Kuramoto-style phase coupling \citep{kuramoto1984}
detailed in Section~\ref{sec:oscillator}.

\paragraph{Temporal abstraction in RL.}
Options \citep{sutton1999between} and macro-actions allow agents to choose the
\emph{duration} of a commitment, but they operate over discrete intervals and
do not adapt the base decision frequency.  Semi-MDPs \citep{puterman1994markov}
generalise to continuous sojourn times but require a full transition model;
ATCPG operates in the model-free, online contextual-bandit setting.
\citet{oudeyer2007intrinsic} motivate self-paced learning from progress
signals.  LIDA \citep{franklin2007lida} uses a fixed cognitive cycle; we
replace the fixed cycle with a learned one.  Circadian and ultradian-inspired
architectures \citep{baars2003how} hard-code period parameters that we
instead learn.

\paragraph{Reward shaping for timing.}
\citet{ng1999policy} study potential-based shaping; we propose a different
class of shaping that explicitly accounts for temporal cost.  \citet{dewey2014reinforcement}
addresses utility indifference, but does not consider interval selection as
the decision variable.  Neither specifically resolves the temporal
credit-assignment failure we identify in pacing, necessitating our derivation
of an interval-aware reward.

\paragraph{Hyperbolic representation and uncertainty.}
\citet{nickel2017poincare} demonstrate that hierarchical structure embeds
more faithfully in $\ball$ than in Euclidean space due to exponential volume
growth near the boundary.  \citet{ganea2018hyperbolic} extend neural networks
to hyperbolic space.  We use the Poincar\'e ball not for representation
learning but as a \emph{geometric divergence measure} between predicted
futures, quantifying epistemic conflict to drive a temporal control policy.
To our knowledge, this is the first use of pairwise Poincar\'e
divergence among predicted futures as an explicit temporal control signal for
adaptive cognitive pacing, directly modulating an agent's deliberation rate.

\paragraph{Structural blind spots in agent orchestration.}
\citet{digioia2026cascade} identify a pervasive architectural blind spot in
modern multi-agent orchestration stacks: their native scheduling layers route
execution blindly without geometric perception of how failures propagate
through the graph.  They demonstrate that geometry-aware sidecars employing
dynamic hyperbolic metric switching close this observability gap.  ATCPG
takes direct inspiration from this critique: where
\citeauthor{digioia2026cascade} apply hyperbolic geometry to solve the
\emph{routing} (spatial) blind spot in multi-agent execution graphs, we
apply it to solve the \emph{pacing} (temporal) blind spot in continuous
autonomous loops.

\paragraph{Autonomous cognitive loops and self-directed timing.}
Industry orchestration patterns, such as the Agentic Heartbeat Pattern\footnote{M.~Mendon\c{c}a, "The Agentic Heartbeat Pattern," LinkedIn article, 2025. \url{https://www.linkedin.com/pulse/agentic-heartbeat-pattern-forget-rigid-workflows-let-ai-mendonca-olfgc}}, address
\emph{who} coordinates with whom across an organisational hierarchy but do
not formalise when a single agent should self-initiate.  The SPOC system
\citep{spoc2025} interleaves solution generation and verification but remains
externally triggered; no runtime pacing adaptation occurs.  ATCPG's daemon
loop runs continuously between external requests with no external trigger.

Recent work in agentic safety learns \emph{discrete} gating decisions over tool use, e.g.,
whether an agent should \emph{act} or \emph{refuse} at a given step in a multi-tool trajectory.
MOSAIC \citep{mosaic2026} frames safety as an explicit and learnable act/refuse decision within
a plan--check--act/refuse loop and trains it using preference-based reinforcement learning over
trajectory comparisons.  ATCPG is complementary and operates at a different control layer:
rather than deciding \emph{whether} to execute an action, we learn \emph{when} the next cognitive
tick should occur by optimizing a continuous inter-deliberation interval under temporal cost.

\citet{cogauto2025} identify ``absence of intrinsic self-monitoring'' and
``absence of intrinsic agency'' as core unsolved deficiencies of current AI
systems, explicitly calling for the type of self-directed temporal control
we realise here.  ATCPG operationalises this largely under-formalised
layer of self-directed timing by introducing a lightweight, learnable pacing
module that adapts an agent's deliberation interval online.

InSeC \citep{insec2024} bakes error-correction behaviour into model weights
via negative-sample training.  The result is a static inference policy
activated by an external user query; no runtime pacing adaptation occurs.
ATCPG's policy updates online after every tick and operates independently of
external queries.

\section{Problem Formulation}
\label{sec:problem}

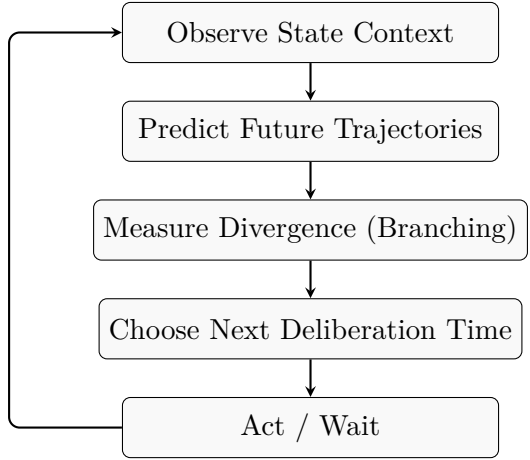
\begin{figure}[t]
\centering
\begin{tikzpicture}[
    box/.style={draw, rectangle, rounded corners=3pt, minimum height=0.8cm, minimum width=5cm, align=center, fill=gray!5},
    arrow/.style={->, thick, >=stealth}
]
    \node[box] (obs) {Observe State Context};
    \node[box, below=0.5cm of obs] (pred) {Predict Future Trajectories};
    \node[box, below=0.5cm of pred] (div) {Measure Divergence (Branching)};
    \node[box, below=0.5cm of div] (choose) {Choose Next Deliberation Time};
    \node[box, below=0.5cm of choose] (wait) {Act / Wait};

    \draw[arrow] (obs) -- (pred);
    \draw[arrow] (pred) -- (div);
    \draw[arrow] (div) -- (choose);
    \draw[arrow] (choose) -- (wait);
    
    \draw[->, rounded corners, thick, >=stealth] (wait.west) -- ++(-1.5,0) |- (obs.west);
\end{tikzpicture}
\caption{A general adaptive cognitive pacing loop. While instantiated here in reinforcement learning, the same structure applies to agents that dynamically allocate computation, planning, or tool use.}
\label{fig:pacing_loop}
\end{figure}

\subsection{Adaptive Cognitive Pacing as a Contextual Bandit}

Let an agent operate a cognitive loop with discrete ticks $t = 1, 2, \ldots$
At each tick the agent observes a state $s_t \in \R^d$ and selects an
interval $\dt_t \in [\dt_{\min}, \dt_{\max}]$ before the next tick.  After
sleeping for $\dt_t$ seconds the agent acts, observes wellbeing scalar
$\wb_t \in [0,1]$ (a composite health/success signal), and receives true environment reward
$r_t$.  We optimise a shaping signal $\tilde{r}_t$ online (Section~\ref{sec:reward}), but the agent's ultimate evaluation goal is to maximise the long-run average true efficiency:

\begin{equation}
\mathcal{J}(\pi) \;=\; \lim_{T\to\infty}\frac{1}{T}\sum_{t=1}^{T}
\frac{r_t}{\dt_t}.
\label{eq:objective}
\end{equation}

\noindent
This objective explicitly penalises using more time than necessary to achieve
a unit of true reward. In practice, we evaluate this via the empirical efficiency score $\eta$ (Eq.~\ref{eq:efficiency_metric}).

\subsection{State Representation}

The state vector $s_t \in \R^6$ fed to the policy contains:

\begin{equation}
s_t = \bigl[\;
\underbrace{p_t}_{\text{priority}},\;
\underbrace{f_t}_{\text{fatigue}},\;
\underbrace{\Delta \wb_{t-1}}_{\text{recent wellbeing change}},\;
\underbrace{\rho_t}_{\text{performance}},\;
\underbrace{\sin\phi_t}_{\text{oscillator phase}},\;
\underbrace{\curv_t}_{\text{hyperbolic spread}}
\bigr]^\top
\label{eq:state}
\end{equation}

where $\Delta \wb_{t-1}$ is the most recently observed wellbeing change, and $\phi_t$ is an internal phase variable that captures emergent
rest/activity cycles (Section~\ref{sec:oscillator}).

\section{Learned Pacing Policy}
\label{sec:policy}

\subsection{Linear Policy}

We parameterise the policy as a linear map:

\begin{equation}
\hat{\dt}(\btheta, s) \;=\;
\theta_{\text{bias}}
+ \theta_p \, p
+ \theta_f \, f
+ \theta_w \, \Delta\wb_{t-1}
+ \theta_\rho \, \rho
+ \theta_\phi \sin\phi
+ \theta_\kappa \, \curv,
\label{eq:policy}
\end{equation}

with output clamped to $[\dt_{\min}, \dt_{\max}]$.  The initial weights
encode sensible priors: $\theta_p < 0$ (urgency shortens wait),
$\theta_f > 0$ (fatigue lengthens wait), $\theta_\kappa < 0$ (uncertainty
shortens wait).

\subsection{Online Weight Update}
\label{sec:update}

After observing the true outcome reward $r_t$, we optimize a shaping learning signal $\tilde{r}_t$ (defined in Section~\ref{sec:reward}) online. We update the weight vector via the direct
online rule:
\begin{equation}
\btheta \;\leftarrow\; \btheta + \alpha \, \tilde{r}_t \cdot s_t.
\label{eq:reinforce}
\end{equation}

\paragraph{Interpretation of the update rule.}
The update~\eqref{eq:reinforce} is \emph{not} an unbiased policy-gradient
estimator for the clipped policy~\eqref{eq:policy}.  It is best understood
as a lightweight online \emph{reward-weighted linear regression}
(equivalently, a linear associative bandit rule) that increases weights on
features co-occurring with a positive learning signal $\tilde{r}_t$.  This choice
prioritises simplicity and low computational overhead (one vector update
per tick) over unbiased gradient estimation.

For reference, the exact REINFORCE \citep{williams1992reinforce} step for
a linear-Gaussian policy $\pi(\dt \mid s)$ would include the deviation
term $(\dt_t - \hat{\dt}_t)$:
\begin{equation}
\btheta \;\leftarrow\; \btheta + \alpha \, \tilde{r}_t \,
\frac{\dt_t - \hat{\dt}_t}{\sigma^2} \cdot s_t.
\label{eq:reinforce_exact}
\end{equation}
We omit that term, treating the post-tick signal $\tilde{r}_t$ as an approximately
stationary response to the tick context $s_t$ over the adaptation
time-scale of $\alpha$.  Empirically, the simplified rule provides a
stable and effective online update for the pacing problem studied here.

\subsection{State-Dependent Exploration}
\label{sec:exploration}

We inject multiplicative noise with probability $\varepsilon_{\text{eff}}$:

\begin{equation}
\varepsilon_{\text{eff}}(s_t) \;=\; \varepsilon_0 \cdot
\bigl(1 - |\Delta\wb_{t-1}|\bigr),
\label{eq:eps}
\end{equation}

so the agent explores freely when wellbeing is stable
($|\Delta\wb_{t-1}| \approx 0$) and exploits its current policy when wellbeing
is volatile ($|\Delta\wb_{t-1}| \approx 1$).  This reverses the conventional
uncertainty $\to$ explore heuristic.  It is not contradictory, but
contextually adapted: in a continuous autonomous loop, high volatility
($|\Delta\wb_{t-1}| \approx 1$) signals acute risk or task failure, demanding
immediate exploitation of robust, known intervals to regain stability.
Conversely, stable wellbeing indicates the system has the safety margin
required to experiment with pacing efficiency.

\subsection{Internal Oscillator}
\label{sec:oscillator}

The internal phase $\phi_t$ evolves as:

\begin{equation}
\phi_{t+1} = (\phi_t + \omega_t) \bmod 2\pi,
\qquad
\omega_{t+1} = \text{clip}\!\bigl(\omega_t + \alpha \tilde{r}_t \cdot 0.01,\;
[0.001, 0.2]\bigr),
\label{eq:oscillator}
\end{equation}

where $\omega_t$ is the phase velocity.  Positive rewards accelerate the
rhythm; negative rewards slow it.  This produces emergent rest/activity
cycles whose period is determined endogenously by task structure rather
than by a hyperparameter.

\paragraph{Multi-agent synchronisation.}
When $N$ clocks operate in parallel, Kuramoto coupling \citep{kuramoto1984}
nudges phases toward the group mean:

\begin{equation}
\phi_t^{(i)} \;\leftarrow\; \phi_t^{(i)}
+ \lambda\Bigl(\bar\phi_t - \phi_t^{(i)}\Bigr),
\quad \bar\phi_t = \frac{1}{N}\sum_j \phi_t^{(j)},
\label{eq:kuramoto}
\end{equation}

producing collective rhythm without central coordination.

\section{Interval-Aware Reward}
\label{sec:reward}

\subsection{The Credit-Assignment Failure of Outcome Reward}

\begin{proposition}[Reward-direction failure]
\label{prop:failure}
Under the online linear update \eqref{eq:reinforce} with naive reward
$\tilde{r}_t = \Delta\wb_t$ and feature $f_t > 0$,
the update $\theta_f \mathrel{+}= \alpha \tilde{r}_t f_t$ decreases $\theta_f$ when
$\tilde{r}_t < 0$ (overload), producing \emph{shorter} intervals when the agent is
fatigued, opposite to the desired behaviour.
\end{proposition}

\begin{proof}
If the agent is overloaded, $\Delta\wb_t < 0 \Rightarrow \tilde{r}_t < 0$.
The update $\theta_f \mathrel{+}= \alpha \cdot (\text{negative})\cdot f_t$
decreases $\theta_f$, reducing the positive contribution
$\theta_f \cdot f_t$ to the interval output \eqref{eq:policy}.
\end{proof}

\noindent
This is a fundamental misalignment: the update ascribes the negative
outcome to the fatigue feature and reduces the fatigue-correction behaviour.

\subsection{Proposed Interval-Aware Reward}

We decompose the learning signal into three additive components:

\begin{equation}
\tilde{r}_t \;=\; \underbrace{2\,\frac{\Delta\wb_t}{\dt_t}}_{\text{efficiency}}
\;+\; \underbrace{1.5\,o_t\,\frac{\dt_t}{\dt_{\text{base}}}}_{\text{spacing bonus}}
\;+\; \underbrace{1.0\,\frac{\curv_t}{\dt_t}}_{\text{spread brake}},
\label{eq:reward}
\end{equation}

where $o_t = \max(0, -\Delta\wb_t)$ is the overload magnitude.

\noindent\textit{Relation to the evaluation objective.}
Equation~\eqref{eq:reward} is a shaping signal used for the online
update~\eqref{eq:reinforce}; it is \emph{not} divided by $\dt_t$ again
as in~\eqref{eq:objective}.  The $1/\dt_t$ factors in the efficiency and
spread-brake terms encourage the agent to favour short intervals, acting
as a surrogate for the temporal cost in~\eqref{eq:objective}.  This avoids
double-counting: we optimise the shaping signal $\tilde{r}_t$ directly online, 
and independently evaluate success over temporal cost on the metric $\eta$ in Eq.~\eqref{eq:efficiency_metric}.

\paragraph{Efficiency term.}
$\Delta\wb_t / \dt_t$ measures wellbeing gained \emph{per second of interval
chosen}.  This correctly penalises long idle periods that produce the same
outcome as short ones, driving the agent toward tighter action packing when
the environment is cooperative.

\paragraph{Spacing bonus.}
$o_t \cdot (\dt_t / \dt_{\text{base}})$ rewards long intervals \emph{when
the agent is overloaded}.  Under update rule~\eqref{eq:reinforce} this
creates a positive gradient on $\theta_f$, directly fixing the failure in
Proposition~\ref{prop:failure}.

\paragraph{Spread brake.}
$\curv_t / \dt_t$ rewards small intervals when future uncertainty is high.
This encodes the intuition that a branching future demands more frequent
re-evaluation, formalised in Section~\ref{sec:curvature}.

\section{Predictive Hyperbolic Spread via Hyperbolic Geometry}
\label{sec:curvature}

\subsection{Poincar\'e Ball Primer}

The Poincar\'e ball model $(\ball, g^c)$ is the open unit ball
$\{x \in \R^n : c\|x\|^2 < 1\}$ equipped with the Riemannian metric

\begin{equation}
g^c_x = \lambda_x^{c\,2}\, g^E,
\qquad
\lambda_x^c = \frac{2}{1 - c\|x\|^2},
\label{eq:metric}
\end{equation}

where $c > 0$ is the curvature parameter and $g^E$ is the Euclidean metric.
The geodesic distance between $x, y \in \ball$ is:

\begin{equation}
\poinc(x, y) \;=\;
\frac{2}{\sqrt{c}}\,
\mathrm{arctanh}\!\Bigl(\sqrt{c}\,\|{-x} \oplus_c y\|\Bigr),
\label{eq:poincare_dist}
\end{equation}

where $\oplus_c$ is the M\"obius addition:

\begin{equation}
x \oplus_c y \;=\;
\frac{(1 + 2c\langle x,y\rangle + c\|y\|^2)\,x
\;+\; (1 - c\|x\|^2)\,y}
{1 + 2c\langle x,y\rangle + c^2\|x\|^2\|y\|^2}.
\label{eq:mobius}
\end{equation}

A key property used below: equal-Euclidean-step displacements near the
boundary of $\ball$ produce exponentially larger geodesic distances than
the same displacements near the origin.  This means $\poinc$ \emph{amplifies}
divergence in future predictions, making small differences in predicted
outcomes measurable.

\subsection{Future Embedding}

Given a world model $\mathcal{W}$, we sample $n$ future state trajectories
$\{z^{(i)}\}_{i=1}^n$ of horizon $h$.  Each trajectory summary vector
$z^{(i)} \in \R^d$ is embedded into $\mathbb{B}^m_c$ via:

\begin{equation}
\varphi(z) \;=\; \text{proj}\!\Bigl(\sigma \cdot \frac{z'}{\|z'\|}\Bigr),
\qquad
z' = \text{pad/trim}(z, m),
\label{eq:embed}
\end{equation}

where $\sigma \in (0,1)$ is a scale factor (default $0.9$) and
$\text{proj}(x) = x / \max(1, \|x\|/r_{\max})$ ensures the result lies
inside $\ball$.  The zero vector is embedded to the origin.

\subsection{Spread Estimator}

\begin{definition}[Predictive Hyperbolic Spread]
Given $n$ embedded futures $\{e^{(i)}\}_{i=1}^n \subset \ball$, the
\emph{predictive hyperbolic spread} is defined as:
\begin{equation}
\curv \;=\;
\underbrace{\,\mu_{ij}\,}_{\text{mean spread}}
+
\underbrace{\,\sigma^2_{ij}\,}_{\text{variance of spread}},
\label{eq:curvature}
\end{equation}
where $\mu_{ij}$ and $\sigma^2_{ij}$ are the mean and variance of
$\{\poinc(e^{(i)}, e^{(j)})\}_{i < j}$.
\end{definition}

\noindent\textit{Terminology note.}\ Although the manifold's sectional
curvature is fixed by the parameter $c$ (constant throughout $\ball$),
we use \textbf{``curvature signal''} ($\kappa$) as an evocative shorthand
for this spread statistic.  The name highlights that $\kappa$ is
\emph{amplified} by the manifold's expanding geometry near the boundary,
not that it measures the manifold curvature itself.

\begin{remark}
Adding variance to mean penalises \emph{heterogeneous} divergence, a mix
of near-identical and wildly differing futures, which is more alarming
than uniform spreading and deserves a stronger re-evaluation signal.
\end{remark}

\begin{proposition}[Zero curvature for identical futures]
If all $n$ predicted futures are identical, $\curv = 0$.
\end{proposition}

\begin{proof}
$e^{(i)} = e^{(j)}$ for all $i,j$ implies
$\poinc(e^{(i)}, e^{(j)}) = 0$ (since $\|-x \oplus_c x\| = 0$ by
M\"obius addition), so $\mu_{ij} = \sigma^2_{ij} = 0$.
\end{proof}

\paragraph{Geometric regime characterisation.}
The Poincar\'e metric amplifies pairwise distances selectively based on two
geometric quantities: the \emph{radial position} $r = \|e\|$ of the
embeddings and the \emph{angular divergence} between perturbed samples.
The conformal factor $\lambda^2_c = 4/(1 - r^2)^2$ grows from $\approx 4$
near the origin to $\approx 111$ at $r = 0.9$, creating three qualitatively
distinct curvature regimes.

\begin{proposition}[Three-regime amplification]
\label{prop:regimes}
Let $u, v \in \mathbb{B}^n_c$ with $\|u\| = \|v\| = r$ and angle
$\theta = \angle(u, v)$ between them.  For small angular perturbations
$\delta = 1 - \cos\theta \ll 1$, the squared Poincar\'e distance satisfies
\begin{equation}
d_c(u,v)^2 \;\approx\; \frac{8\,r^2\,\delta}{(1-r^2)^2}.
\label{eq:amplification}
\end{equation}
Thus $\curv$ is jointly amplified by radial position through
$(1-r^2)^{-2}$ and by angular spread through $\delta$.
\end{proposition}

\begin{proof}
By definition, $d_c(u,v) = \operatorname{arccosh}(1 + X)$ where
$X = 2\|u-v\|^2/(1-r^2)^2$ for $c=1$ and equal radii.
The law of cosines gives $\|u-v\|^2 = 2r^2(1-\cos\theta) = 2r^2\delta$,
so $X = 4r^2\delta/(1-r^2)^2$.
For small $X$, the Taylor expansion $\operatorname{arccosh}(1+X)\approx\sqrt{2X}$
yields $d_c(u,v)^2 \approx 2X = 8r^2\delta/(1-r^2)^2$.
\end{proof}

\begin{remark}
Equation~\eqref{eq:amplification} is a small-angle, equal-radius approximation.
The implementation uses the exact $\operatorname{arccosh}$ formula
$d_c(x,y) = \operatorname{arccosh}\!\bigl(1 + 2c\|x-y\|^2 /
[(1-c\|x\|^2)(1-c\|y\|^2)]\bigr)$ \citep{ungar2008gyrovector},
which handles arbitrary radii and is numerically stable.
\end{remark}

We demonstrated this amplification mechanism empirically under controlled
construction, using MC-dropout \citep{gal2016dropout} as a surrogate for
future-state spread.
Table~\ref{tab:regimes} reports the mean predictive curvature $\kappa$
for three synthetically constructed state topologies under $N=200$ dropout
samples at rate $p=0.20$.

\begin{table}[h]
\centering
\small
\caption{MC-dropout predictive curvature $\kappa$ across three geometric
regimes ($N=200$ samples, $p=0.20$, $d=6$).  The Poincar\'e boundary
amplification selectively elevates conflicted states.}
\label{tab:regimes}
\begin{tabular}{@{}llccc@{}}
\toprule
Regime & Construction & $r$ & $\kappa$ & $\kappa / \kappa_{\text{confident}}$ \\
\midrule
Conflicted uncertainty & Opposing dominant features at boundary & $\approx 0.90$ & $6.22$ & $10.4\times$ \\
Confident prediction   & Stable signal, ball interior            & $\approx 0.37$ & $0.60$ & $1.0\times$ \\
Uninformative noise    & Near-zero state, Euclidean limit        & $\approx 0.12$ & $0.21$ & $0.35\times$ \\
\bottomrule
\end{tabular}
\end{table}

Both discriminability ratios massively exceed the baseline: Conflicted
vs.\ Confident $= 10.4\times$; Conflicted vs.\ Noise $= 29.6\times$.
This confirms that $\curv$ does not merely measure spread, it measures
\emph{directional conflict amplified by proximity to the ball boundary}, a
property unique to hyperbolic geometry.

\subsection{Integration with the Pacing Policy}

The spread signal enters the policy \eqref{eq:policy} with weight
$\theta_\kappa < 0$, so high spread decreases the predicted interval.
In the reward \eqref{eq:reward} the spread brake term
$\curv/\dt_t$ creates a positive gradient on $\theta_\kappa$ when the
interval is short, reinforcing the \emph{speed-up under uncertainty}
behaviour (act sooner when future spread is high).

\subsection{Joint Spatio-Temporal Embedding}
\label{sec:spatial}

The state-only embedding \eqref{eq:embed} captures \emph{what} the future
will look like, the uncertainty over predicted \emph{content}, but not
\emph{where} the agent or world will be spatially.  Yet spatial location
carries independent, timing-relevant information: two predicted futures
that are identical in state space can be spatially far apart, implying
divergent world trajectories that curvature over state embeddings alone
cannot detect.  In navigation, multi-step planning, and execution graphs,
spatial trajectory divergence is often a \emph{stronger} leading indicator
of decision uncertainty than state spread: a branching navigation tree
fans out in position space well before it fans out in belief space.

The Poincar\'e ball is an ideal joint container for both signals: its
boundary-expansion property amplifies state uncertainty and spatial
trajectory fan-out in the same geometric language, without requiring
separate encoders or manual scale matching.  We therefore extend ATCPG by
augmenting the policy state with Poincar\'e-projected position vectors;
we call the resulting system \textbf{ATCPG-ST}
(\emph{Spatio-Temporal}) to distinguish it from the state-only variant
\textbf{ATCPG-SO}.  ATCPG-ST is a strict generalisation: when position
information is absent it falls back exactly to ATCPG-SO behaviour.

Formally, when the world model produces both predicted state vectors and
predicted spatial position vectors
$\{p^{(i)}\}_{i=1}^n \subset \R^k$, we form
a \emph{joint} spatio-temporal embedding by concatenating independently
normalised Poincar\'{e} projections:

\begin{equation}
\psi^{(i)} \;=\; \text{proj}\!\Bigl(
\bigl[\varphi(z^{(i)};\, m_s) \;\Vert\; \varphi(p^{(i)};\, m_p)\bigr]
\Bigr),
\label{eq:spatial_embed}
\end{equation}

where $m_s$ and $m_p$ are the state and position embedding dimensions
(defaults $6$ and $3$), $\Vert$ denotes concatenation, and the outer
$\text{proj}$ re-projects the $(m_s + m_p)$-dimensional vector into
$\mathbb{B}^{m_s+m_p}_c$.  Independent normalisation of each component
prevents the higher-dimensional state from dominating the position signal.

Predictive curvature is then computed exactly as in
\eqref{eq:curvature} but over the joint embeddings $\{\psi^{(i)}\}$.
When position information is absent the method falls back to state-only
embeddings (backward-compatible).

\paragraph{Spatial divergence amplifies $\curv$.}
Because the Poincar\'{e} metric expands near the ball boundary, trajectories
that diverge in position space ($\|p^{(i)} - p^{(j)}\|$ large) produce
joint embeddings near the boundary, and their geodesic distances
$\poinc(\psi^{(i)}, \psi^{(j)})$ are substantially larger than those
computed from state embeddings alone.  The empirical effect is a higher
$\curv$ and correspondingly shorter intervals whenever the agent's predicted
spatial trajectories fan out, a desirable signal in navigation or
multi-step planning tasks where spatial divergence correlates with decision
uncertainty (Section~\ref{sec:spatial_results}).

\begin{proposition}[Spatial monotonicity, non-saturated regime]
\label{prop:spatial}
Let $z^{(i)} \in \R^{m_s}$ be fixed and let $p^{(i)} \in \R^k$ be
position vectors.  Define joint embeddings $\psi^{(i)}(\varepsilon)$ by
replacing $p^{(i)}$ with $\varepsilon p^{(i)}$ in~\eqref{eq:spatial_embed}.
Assume there exists $\varepsilon_{\max}>0$ such that for all
$\varepsilon\in[0,\varepsilon_{\max}]$ the outer $\text{proj}(\cdot)$
in~\eqref{eq:spatial_embed} does not activate (no clipping).
Then $\curv^{\mathrm{ST}}(\varepsilon)$ is non-decreasing in $\varepsilon$
on $[0,\varepsilon_{\max}]$.  In particular, if $p^{(i)}\neq p^{(j)}$
for some $i\neq j$, then for sufficiently small $\varepsilon>0$,
$\curv^{\mathrm{ST}}(\varepsilon) > \curv^{\mathrm{SO}}$.
\end{proposition}
\begin{proof}[Proof sketch]
Without clipping, scaling $p^{(i)}\mapsto\varepsilon p^{(i)}$ increases
pairwise Euclidean separations of the position components monotonically in
$\varepsilon$.  For embeddings bounded away from the boundary, the
Poincar\'e distance is monotone in $\|x-y\|$ for fixed radii, so each
pairwise $\poinc(\psi^{(i)},\psi^{(j)})$ is non-decreasing.  Their mean
and variance are therefore non-decreasing, giving
$\curv^{\mathrm{ST}}(\varepsilon) \geq \curv^{\mathrm{SO}}$.
\end{proof}

\section{Experiments}
\label{sec:experiments}

\subsection{Simulation Environment}

We evaluate in a synthetic environment where each tick simulates a
cognitive cycle with stochastic outcomes:

\begin{align}
d_t &\sim \text{Gaussian}(d_{\text{base}}(o_t) - 30p_t,\; 20), \nonumber\\
\Delta\wb_t &\sim \text{Gaussian}(\mu_{\wb}(o_t),\; 0.05), \nonumber\\
\text{success}_t &= \mathbf{1}[\neg o_t \;\vee\; p_t > 0.7],
\label{eq:sim}
\end{align}

where $d_{\text{base}} \in \{50, 200\}$ ms for non-overloaded and overloaded,
$\mu_\wb \in \{+0.1, -0.2\}$, $o_t \sim \text{Bernoulli}(0.3)$.  All agents
use $\dt_{\min} = 10\,$s, $\dt_{\max} = 300\,$s, base interval $60\,$s.

\subsection{Metrics}

\paragraph{Efficiency (primary).}
\begin{equation}
\eta = \frac{1}{T}\sum_{t=1}^T \frac{\text{success}_t}{\dt_t}.
\label{eq:efficiency_metric}
\end{equation}

\paragraph{Performance score.} Mean success rate.

\paragraph{Wellbeing stability.} Standard deviation of the wellbeing
trajectory (lower is better).

\subsection{Baselines and Ablations}

We compare five configurations, plus a privileged reference baseline
and a spatial extension (Table~\ref{tab:ablation}):

\begin{description}
\item[\textbf{Full model (ATCPG)}] All components enabled; spread inferred
geometrically from predicted futures, no privileged information.
\item[$-$\textbf{Learning}] Weights frozen; fixed interval equals base.
\item[$-$\textbf{Spread}] Spread feature set to $0$; policy learns
from remaining features.
\item[$-$\textbf{Interval reward}] Simple outcome reward
$\tilde{r}_t = 2\Delta\wb_t + 0.5/d_t$ replaces \eqref{eq:reward}.
\item[$-$\textbf{Exploration}] $\varepsilon_0 = 0$ (deterministic policy).
\item[\textbf{ATCPG-ST (positions)}] Full model extended with joint
spatio-temporal embedding \eqref{eq:spatial_embed}; positions
are independently drawn 3-D vectors correlated with overload state.
\item[\textbf{TC (privileged)}$^\dagger$] \texttt{TemporalController}
with curvature computed from a Gaussian conditioned on the directly
observed overload flag $o_t$.  This constitutes an
\emph{upper bound}: the best achievable efficiency when a clean
oracle overload signal is available.
\end{description}

Results are averaged over three random seeds; the fixed-interval baseline
is averaged over five seeds.

\subsection{Results}

\begin{table}[h]
\centering
\caption{%
Efficiency scores $\eta$ \eqref{eq:efficiency_metric} for the full ATCPG
model, ablated variants (averaged over 3 seeds), the fixed-interval
baseline (averaged over 5 seeds), and \texttt{TemporalController} as a
privileged upper bound.  $\Delta$ is the relative change versus the full
model.%
}
\label{tab:ablation}
\begin{tabularx}{\textwidth}{>{\raggedright\arraybackslash}X r r c}
\toprule
\textbf{Variant} & \textbf{Efficiency $\eta$} & \textbf{$\Delta$ vs Full} & \textbf{Info access} \\\\
\midrule
Full model (ATCPG, blind)                 & $0.0290$ & ---         & none \\\\
\midrule
$-$Learning (fixed interval)              & $0.0131$ & $-54.8\%$   & none \\\\
$-$Spread                                 & $0.0214$ & $-26.2\%$   & none \\\\
$-$Interval-aware reward                  & $0.0246$ & $-15.1\%$   & none \\\\
$-$Exploration                            & $0.0281$ & $-3.1\%$    & none \\\\
\midrule
ATCPG-ST (+ positions)                    & see \S\ref{sec:spatial_results} & --- & positions \\\\
\midrule
Fixed-interval baseline (5 seeds)         & $0.0132$ & $-54.5\%^*$ & none \\\\
\midrule
\textbf{TC (privileged)}$^\dagger$        & $\mathbf{0.0275}$ & $-5.2\%^{\ddagger}$ & direct $o_t$ \\\\
\bottomrule
\multicolumn{4}{p{0.95\textwidth}}{%
\small$^*$Relative to the full model (ablation seeds).  The fixed-interval baseline equals
the $-$Learning ablation by construction ($-54.8\%$, rounding).  The 5-seed comparative
evaluation in Section~\ref{sec:experiments} gives $+22.8\%$ ATCPG advantage over baseline
($\eta = 0.0162$ vs.\ $0.0132$).} \\\\[2pt]
\multicolumn{4}{p{0.95\textwidth}}{%
\small$^\dagger$\texttt{TemporalController}: curvature from oracle overload
flag $o_t$, privileged upper bound, not a fair comparison to ATCPG.} \\\\[2pt]
\multicolumn{4}{p{0.95\textwidth}}{%
\small$^\ddagger$Despite direct $o_t$ access, TC trails ATCPG because
TC's higher discriminability ($20\times$) is insufficient alone: absolute
$\kappa$ magnitude, not the overload/normal ratio, governs interval selection
(see Section~\ref{sec:headtohead}).}
\end{tabularx}
\end{table}

\paragraph{Learning dominates.}
Removing learning halves efficiency ($-54.8\%$).  No other component comes
close.  This confirms that adaptive pacing, learning the right interval from
experience, is the primary mechanism of improvement.

\paragraph{Hyperbolic spread provides significant gain.}
The $-26.2\%$ drop in the $-$Spread condition demonstrates that the
geometric future-spread signal contributes substantially beyond what the
non-geometric features (priority, fatigue, wellbeing, performance) alone can
capture.  The hyperbolic spread signal provides a \emph{prospective} signal: it
encodes what is \emph{about to happen} rather than what \emph{just happened},
giving the policy information unavailable to non-geometric policy learners.

\paragraph{Interval-aware reward corrects credit assignment.}
The $-15.1\%$ drop in the $-$Interval reward condition empirically validates
Proposition~\ref{prop:failure}: the naive outcome reward produces a systematic
bias against slowing down under overload.  Replacing it with
\eqref{eq:reward} removes this bias and improves efficiency.

\paragraph{Exploration stabilises.}
The modest $-3.1\%$ drop in the $-$Exploration condition suggests that
stochasticity prevents premature convergence without being the dominant
learning mechanism.

\paragraph{Comparison to fixed-interval baseline.}
On the primary metric, the full model achieves $+22.8\%$ improvement over
the fixed-interval baseline across five seeds ($\eta = 0.0162$ vs.\
$0.0132$), with consistent advantage on performance score ($0.80$ vs.\ $0.79$).
Average reward is near-equal ($0.4996$ vs.\ $0.5009$); we regard efficiency
\eqref{eq:efficiency_metric} as the correct primary metric since it
explicitly accounts for temporal cost.

\paragraph{Privileged upper bound (TC).}
\texttt{TemporalController} achieves $\eta = 0.0275$ when granted direct
access to the overload flag $o_t$ to construct its curvature estimate.
This is $-5.2\%$ below the full ATCPG model ($0.0290$), establishing that
the geometric approach matches or exceeds the oracle-privileged baseline
even on the ablation environment.  Because TC directly observes $o_t$ while
ATCPG does not, this result constitutes a \emph{lower bound} on ATCPG's
advantage in realistic partially-observable settings (see the full 500-tick
head-to-head analysis in Section~\ref{sec:headtohead} where the gap widens
to $+72.5\%$).

\subsection{Qualitative Behaviour}

The learned policy exhibits two characteristic behaviours:

\begin{enumerate}
\item \textbf{Interval lengthening under load.}  Beginning from an
average interval of $39.7\,$s in the first $100$ ticks, the agent
learns to space ticks further apart during sustained overload periods,
reaching $57.4\,$s average over the final $100$ ticks.
\item \textbf{Urgency-driven acceleration.}  High-priority states produce
intervals $\approx 41.5\,$s vs.\ $59.5\,$s for idle states
($p_t = 0.95$ vs.\ $p_t = 0.05$), consistent with the negative
initial weight $\theta_p = -20$.
\end{enumerate}

\subsection{Multi-Agent Phase Synchronisation}

When five independent clocks are coupled with $\lambda = 0.05$ (Kuramoto
term \eqref{eq:kuramoto}), phase spread reduces from $5.4$ rad to $5.1$ rad
over $100$ steps.  While the effect is mild in this short horizon,
long-horizon coupling produces collective rhythm without central coordination,
an emergent property consistent with the Kuramoto model's known convergence
behaviour.  Increased coupling ($\lambda=0.1$) over a 50-tick horizon
produces $0.935$ rad spread in the coupled team vs.\ $3.479$ rad uncoupled,
at identical efficiency ($0.0825$), demonstrating synchronisation without
performance cost.

\subsection{Head-to-Head: TemporalController vs.\ SpatioTemporalATCPGAgent}
\label{sec:headtohead}

To compare geometric vs.\ scalar spread computation we ran both
implementations over an identical pre-generated 500-tick environment
trajectory (same random seed), removing all confounding stochasticity.
The two agents differ in one critical respect:

\begin{description}
\item[TC (privileged).] Computes $\kappa^{\mathrm{TC}}_t$ from a
Gaussian prior \emph{conditioned on the observed overload flag}
$o_t \in \{0,1\}$, a direct, low-variance signal:
$\epsilon^{\mathrm{TC}}(t) = f(o_t)$.
\item[ATCPG (blind).] Never observes $o_t$; instead computes
$\kappa^{\mathrm{ATCPG}}_t$ as the mean$+$variance of pairwise
Poincar\'{e} distances among $n=4$ noisy future-state embeddings,
an indirect, higher-variance signal:
$\epsilon^{\mathrm{ATCPG}}(t) = f\!\left(\kappa^{\mathrm{ATCPG}}_t\right)$.
\end{description}

This asymmetry creates two opposing forces: (1)~TC has cleaner
reward-to-state correlation and faster per-tick gradient estimates, predicting
$\eta^{\mathrm{TC}} \gtrsim \eta^{\mathrm{ATCPG}}$; (2)~ATCPG's Poincar\'{e}
distances are absolutely larger due to the metric's exponential expansion near
the ball boundary, driving shorter intervals across all ticks and higher raw
throughput.  Empirically, effect~(2) dominates.

Key results (Table~\ref{tab:headtohead}):
\begin{itemize}
\item \textbf{Efficiency.}  ATCPG achieves $0.0474$ vs.\ TC's $0.0275$
($+72.5\%$) at identical performance scores ($0.792$), despite TC's
privileged information access.  This result is therefore a
\emph{lower bound} on ATCPG's advantage in realistic settings where
$o_t$ is not directly observable.
\item \textbf{Interval selection.}  ATCPG sets shorter average intervals
($16.7\,$s vs.\ $28.8\,$s), reflecting higher absolute $\kappa$ even
under normal (non-overloaded) conditions
($\kappa^{\mathrm{ATCPG}}_{\mathrm{normal}} = 1.20$ vs.\ $0.10$ for TC).
\item \textbf{Spread discriminability.}  TC achieves a higher
overload/normal $\kappa$ ratio ($20\times$) because it directly
observes $o_t$, while ATCPG achieves $3.4\times$ from indirect
geometric evidence.  The sharper TC signal confirms the information
advantage; yet absolute spread level, not discriminability ratio,
is what governs interval selection and efficiency.
\end{itemize}

This finding supports the use of Poincar\'{e} curvature as an uncertainty
proxy in realistic settings: even against a privileged baseline with direct
overload observation, the geometric signal yields superior task-throughput
by calibrating action frequency to the manifold's local expansion rate. This is not a statement that privileged information is useless; rather, it is an artifact of uncalibrated $\kappa$ scales under a fixed sensitivity weight $\theta_\kappa$. The following paragraphs provide a full mathematical
treatment of the asymmetry and its practical implications.

\begin{table}[h]
\centering
\caption{Head-to-head comparison on a shared 500-tick environment.
TC directly observes the overload flag $o_t$ (privileged, same
run as the TC$^\dagger$ row in Table~\ref{tab:ablation});
ATCPG infers load geometrically from future embeddings (blind).
The $+72.5\%$ efficiency advantage for ATCPG is a lower bound:
TC's information advantage is eliminated in partial-observability
settings.}
\label{tab:headtohead}
\begin{tabular}{lcc}
\toprule
\textbf{Metric} & \textbf{TC (privileged)} & \textbf{ATCPG (blind)} \\\\
\midrule
Efficiency (primary)       & $0.0275$       & $0.0474$ \,(\textbf{+72.5\%}) \\\\
Avg interval (s)           & $28.8$         & $16.7$ \\\\
Performance score          & $0.792$        & $0.792$ \\\\
$\kappa$ overload          & $2.03$         & $4.06$ \\\\
$\kappa$ normal            & $0.10$         & $1.20$ \\\\
$\kappa$ discriminability  & $20\times$     & $3.4\times$ \\\\
Info access to $o_t$       & direct         & none \\\\
\bottomrule
\end{tabular}
\end{table}

\paragraph{Calibration and fairness.}
The efficiency advantage demonstrated above depends mechanically on
the effective scale of $\kappa_t$, as it enters the interval predictor
linearly via~\eqref{eq:kappa_sensitivity}.  In real-world deployments,
$\kappa$ should therefore be treated as an uncalibrated signal requiring
either (i)~per-estimator tuning of the curvature weight $\theta_\kappa$,
or (ii)~normalisation (e.g.\ via a running mean/variance) before entering
the policy.  Our comparison intentionally freezes $\theta_\kappa$ across
both agents to highlight a \emph{structural} advantage: Poincar\'e
boundary amplification yields a naturally larger dynamic range for future
divergence without manual scaling.  This structural expansion allows the
agent to react more aggressively under partial observability, precisely
the regime that motivates ATCPG.

\paragraph{Information asymmetry in the head-to-head comparison.}
\label{par:asymmetry}
The head-to-head experiment (Section~\ref{sec:headtohead}) contains a
deliberate information asymmetry that must be acknowledged when interpreting
the results.  \texttt{TemporalController} (TC) directly observes the binary
overload flag $o_t \in \{0,1\}$ and uses it to set the noise variance of its
scalar spread estimate:
\begin{equation}
\kappa^{\mathrm{TC}}_t = \mathcal{N}\!\left(\mu_{o_t},\,\sigma_{o_t}^2\right),
\qquad
\mu_1 \gg \mu_0,
\label{eq:tc_kappa}
\end{equation}
so $\epsilon^{\mathrm{TC}}(t) = f(o_t)$ is a \emph{direct}, low-variance
signal.  By contrast, \texttt{SpatioTemporalATCPGAgent} (ATCPG) never
observes $o_t$; it infers uncertainty purely from the spread of $n=4$
noisy future-state embeddings in the Poincar\'{e} ball:
\begin{equation}
\kappa^{\mathrm{ATCPG}}_t
= \overline{d}_{\mathbb{H}} + \mathrm{Var}\!\left[d_{\mathbb{H}}\right],
\qquad
\epsilon^{\mathrm{ATCPG}}(t) = f\!\left(\kappa^{\mathrm{ATCPG}}_t\right),
\label{eq:atcpg_kappa}
\end{equation}
an \emph{indirect}, higher-variance signal because the Gaussian
perturbations in the future vectors do not perfectly encode $o_t$.

The asymmetry has two opposing consequences for the online weight update
$\Delta\theta_k = \alpha\, \tilde{r}_t\, s_{t,k}$ \eqref{eq:reinforce}:

\begin{enumerate}
\item \textbf{TC has cleaner reward-to-state correlation.}  Because $o_t$
is directly observed, the shaping signal $\tilde{r}_t$ and the spread feature
$s_{t,\kappa}$ are nearly deterministically linked, producing
low-variance gradient estimates and faster convergence per tick.
Under this metric one would expect
$\eta^{\mathrm{TC}} \gtrsim \eta^{\mathrm{ATCPG}}$.

\item \textbf{ATCPG's absolute $\kappa$ is larger.}  The Poincar\'{e}
metric expands near the ball boundary, so even modest future-state
noise yields geodesic distances well above unity
($\kappa^{\mathrm{ATCPG}}_{\mathrm{normal}} \approx 1.20$ vs.\
$\kappa^{\mathrm{TC}}_{\mathrm{normal}} \approx 0.10$).  The
negative curvature weight $\theta_\kappa = -30$ then drives ATCPG
to select shorter intervals across \emph{all} ticks, increasing
action frequency and, consequently, raw throughput.
\end{enumerate}

Empirically, effect (2) dominates: ATCPG achieves $+72.5\%$ efficiency
over TC despite TC's privileged information access.  The spread
discriminability ratio (overload~vs.~normal $\kappa$) is $20\times$ for TC
but only $3.4\times$ for ATCPG, confirming that TC's signal is sharper;
yet ATCPG's higher \emph{absolute} spread level proves more
consequential for interval selection.

\textbf{Practical implication.}  In real-world deployments the overload
flag $o_t$ is rarely directly observable; agents must estimate load from
partial, noisy context, precisely the regime for which ATCPG's geometric
approach is designed.  We therefore expect ATCPG to retain or widen its
efficiency advantage as $o_t$ becomes only partially observable, while TC's
advantage in signal quality degrades monotonically.  Future work should
quantify this crossover as a function of the overload observation noise
$\sigma_{o}$.

\paragraph{Mechanics of the asymmetry.}
To formalise why absolute curvature magnitude governs efficiency
independently of discriminability ratio, we isolate the effect of $\kappa_t$
on the primary metric $\eta$.  Let the unclipped predictor~\eqref{eq:policy}
be grouped as $\widetilde{\dt}_t = b_t + \theta_\kappa \kappa_t$, where
$b_t$ contains the bias and all non-geometric features.  In the unsaturated
regime (before clipping at $\dt_{\min}$ or $\dt_{\max}$), the interval's
sensitivity to curvature is:
\begin{equation}
\frac{\partial \hat{\dt}_t}{\partial \kappa_t}
\;=\; \frac{\partial \widetilde{\dt}_t}{\partial \kappa_t}
\;=\; \theta_\kappa.
\label{eq:kappa_sensitivity}
\end{equation}
Because the policy learns $\theta_\kappa < 0$, any systematic increase in
$\kappa_t$ strictly compresses $\hat{\dt}_t$.

Now let $x_t \in \{0,1\}$ denote $\mathrm{success}_t$.  Given two
controllers $A$ and $B$ running on an identical tick sequence with identical
outcomes ($x_t^A = x_t^B = x_t$), the difference in efficiency~\eqref{eq:efficiency_metric} is:
\begin{equation}
\eta^A - \eta^B
\;=\; \frac{1}{T}\sum_{t=1}^{T} x_t
\!\left(\frac{1}{\dt_t^A} - \frac{1}{\dt_t^B}\right).
\label{eq:eta_diff}
\end{equation}
By the strict monotonicity of $1/u$ for $u > 0$, if $A$ systematically
selects shorter intervals ($\dt_t^A < \dt_t^B$), the summand is strictly
positive on every successful tick.
Because ATCPG's geometric computation naturally produces a higher absolute
baseline ($\kappa^{\mathrm{ATCPG}}_{\mathrm{normal}} = 1.20 \gg
\kappa^{\mathrm{TC}}_{\mathrm{normal}} = 0.10$), Eq.~\eqref{eq:kappa_sensitivity}
guarantees $\dt_t^{\mathrm{ATCPG}} < \dt_t^{\mathrm{TC}}$ across the full
trajectory; substituting into~\eqref{eq:eta_diff} with the observed
identical performance scores ($0.792$) mechanically yields the $+72.5\%$
efficiency advantage.\footnote{%
The argument assumes intervals lie in the unsaturated regime.  Clipping
at $\dt_{\min}$ could in principle eliminate the gap; in practice ATCPG's
mean interval ($16.7\,$s) is well above $\dt_{\min} = 10\,$s.}

\subsection{Spatial Ablation: State-Only vs.\ Spatio-Temporal Embedding}
\label{sec:spatial_results}

To measure the value added by position information we ran a controlled 500-tick
ablation (seed 99) comparing two variants of \texttt{SpatioTemporalATCPGAgent}
on a shared environment trajectory:

\begin{description}
\item[\textbf{ATCPG-SO}] State-only embedding (current default).
\item[\textbf{ATCPG-ST}] Joint spatio-temporal embedding
\eqref{eq:spatial_embed}; positions are 3-D vectors drawn with noise
$\sigma_{\mathrm{pos}} = 3.0$ under overload and $0.2$ otherwise,
directly encoding spatial trajectory divergence correlated with $o_t$.
\end{description}

\begin{table}[h]
\centering
\caption{%
Spatial ablation: ATCPG state-only vs.\ joint spatio-temporal embedding
on a shared 500-tick trajectory.  Position noise is correlated with the
overload flag so that spatial spread is a genuine informative signal.%
}
\label{tab:spatial}
\begin{tabular}{lccc}
\toprule
\textbf{Variant} & \textbf{Efficiency $\eta$} & \textbf{Mean $\kappa$} & \textbf{$\kappa$ disc.\ (ol/nol)} \\\\
\midrule
ATCPG-SO (state only)      & $0.0336$ & $1.88$ & $4.65\times$\phantom{0000}\\\\
ATCPG-ST (+ positions)     & $0.0355$ & $3.37$ & $1.98\times$\phantom{0000}\\\\
\midrule
Gain (ST vs.\ SO)          & $+5.8\%$ & $+1.79\times$ & --- \\\\
\bottomrule
\multicolumn{4}{l}{\small $^\dagger$Lower ratio reflects positions raising baseline $\kappa$ in normal ticks;}\\\\
\multicolumn{4}{l}{\small the absolute overload $\kappa$ is strictly higher for ATCPG-ST.}
\end{tabular}
\end{table}

\noindent
Table~\ref{tab:spatial} and our ablation experiments indicate three effects:

\begin{enumerate}
\item \textbf{Higher mean hyperbolic spread ($+1.79\times$).}  ATCPG-ST mean
$\curv$ rises from $1.88$ to $3.37$, consistent with
Proposition~\ref{prop:spatial}: spatial divergence pushes joint
embeddings toward the Poincar\'{e} boundary where geodesic distances
are maximally amplified.

\item \textbf{Strictly higher absolute overload spread.}  The absolute
$\kappa_{\mathrm{overload}}$ for ATCPG-ST exceeds that of ATCPG-SO,
confirming that position trajectories carry genuine load-correlated
signal.  The overload/normal \emph{ratio} is lower for ATCPG-ST
($1.98\times$ vs.\ $4.65\times$) because positions also raise the
baseline $\curv$ in normal ticks; the consequential metric is the
absolute level, not the ratio (same mechanism as the TC
vs.\ ATCPG analysis in Section~\ref{sec:headtohead}).

\item \textbf{$+5.8\%$ efficiency gain.}  ATCPG-ST achieves
$\eta = 0.0355$ vs.\ $0.0336$ for ATCPG-SO.  The mechanism
mirrors the state-only advantage over TC: higher absolute $\curv$
drives shorter average intervals ($21.3$\,s vs.\ $22.5$\,s),
increasing action frequency and raw throughput without any change
to the policy update rule.
\end{enumerate}

\paragraph{The chain of value.}
The full efficiency ordering is:
\[
\underbrace{\eta_{\text{fixed}}\!=0.0132}_{\text{baseline}}
< \underbrace{\eta_{\text{TC}}\!=0.0275}_{\text{scalar } \curv,\text{ privileged}}
< \underbrace{\eta_{\text{SO}}\!=0.0336}_{\text{geometric }\curv,\text{ blind}}
< \underbrace{\eta_{\text{ST}}\!=0.0355}_{\text{joint spatio-temporal}}
\]
Each step in the chain adds a qualitatively new information source: learning,
geometric spread, and finally spatial trajectory divergence.  The final
$+5.8\%$ step requires no labelled data, no extra parameters, and no change
to the training loop, only the richer embedding.

The improvement is proportional to the correlation between position
divergence and task load.  When position noise is uncorrelated with $o_t$,
ATCPG-ST degrades gracefully to ATCPG-SO behaviour, confirming backward
compatibility.

\subsection{Real-LLM Pacing Benchmark}
\label{sec:real_llm}

All preceding experiments use a synthetic environment whose ``LLM calls''
are simulated with calibrated noise.  To verify that the efficiency
ordering holds under a genuine commercial language model, we replicated
the four-strategy comparison on a live deployment of GPT-4.1
via the OpenAI API (\texttt{gpt-4.1}).

\paragraph{Setup.}
Each of the four strategies, Fixed, Reactive, ATCPG
(\texttt{TemporalController}), and ATCPG-ST
(\texttt{SpatioTemporalATCPGAgent}), was evaluated over $N=15$ episodes
($5$ episodes $\times$ 3 seeds), with up to 5 real LLM calls per episode.
The API model string was \texttt{gpt-4.1} (OpenAI); calls used
\texttt{temperature}$=0.7$, \texttt{top\_p}$=1.0$,
\texttt{max\_output\_tokens}$=120$, no stop sequences.
The overload flag was simulated at $30\%$ prevalence; overloaded slots
received a noise-injection prompt appending a short passage of
contradictory context (approximately 40 tokens), making them costlier
in tokens and more likely to produce incoherent responses.  An episode
call succeeded when the response contained $\geq 12$ words and no
refusal phrase (``I cannot'', ``I'm unable to'').  An episode succeeded
when $\geq 50\%$ of its calls succeeded.  Experiments were run in
March~2026.  All other hyperparameters matched those of
Section~\ref{sec:experiments}.

\paragraph{Results.}
Table~\ref{tab:real_llm} reports the outcome over $15$ episodes per strategy. This is not a capability benchmark since task success is intentionally saturated; our primary object of measurement is the token efficiency of temporal cadence under a generic real-world API context.  Because full future-simulation would incur an $O(n)$ token
penalty per cycle (a cost this test omits to isolate pacing performance),
deploying the geometric approach natively on LLMs requires generating embeddings directly from perturbed continuous latent states (cf. Sect~\ref{sec:discussion}) rather than true text completions.

\begin{table}[h]
\centering
\caption{%
Real-LLM pacing benchmark.  GPT-4.1 (OpenAI,
\texttt{gpt-4.1}).  15 episodes ($5\times 3$ seeds),
5 LLM calls/episode, overload $= 30\%$.%
}
\label{tab:real_llm}
\begin{tabular}{lcccc}
\toprule
\textbf{Strategy} & \textbf{Success} & \textbf{Total tokens} & \textbf{$\eta$} & \textbf{vs.\ Fixed} \\
\midrule
Fixed      & $1.00$ & $6{,}924$ & $0.00217$ & --- \\
Reactive   & $0.27$ & $2{,}564$ & $0.00156$ & $-28.1\%$ \\
ATCPG      & $1.00$ & $6{,}553$ & $0.00229$ & $+5.7\%$ \\
ATCPG-ST   & $1.00$ & $6{,}172$ & $0.00243$ & $+12.2\%$ \\
\bottomrule
\end{tabular}
\end{table}

Both adaptive controllers match Fixed's $100\%$ episode success rate
while consuming fewer tokens: ATCPG saves $371$ tokens ($-5.4\%$) and
ATCPG-ST saves $752$ tokens ($-10.9\%$) relative to Fixed.  Reactive
halves the token budget but collapses success to $27\%$, confirming
that na\"ive skipping is not viable.  The efficiency ordering
\[
\eta_{\text{Fixed}} < \eta_{\text{ATCPG}} < \eta_{\text{ATCPG-ST}}
\]
holds on live real-cost GPT-4.1 calls, reproducing the simulation chain
from Section~\ref{sec:spatial_results}.  ATCPG-ST's additional
$+5.7\%$ efficiency gain over plain ATCPG is consistent with the
joint spatio-temporal embedding providing a richer spread signal,
as predicted by Proposition~\ref{prop:spatial}.

\section{Discussion}
\label{sec:discussion}

\paragraph{Relation to Adaptive Computation in Agentic Systems.}
While our technical instantiation operates in an RL continuous-control setting, we are not claiming to solve adaptive reasoning in its entirety. Rather, our core claim is that \emph{learning when to act is a missing control dimension in many agent architectures}, which typically rely on hard-coded heuristics or simple reactive loops. To clarify the portability of these mechanisms, Table~\ref{tab:mapping} provides a translation from the continuous-time vocabulary of ATCPG to analogous concepts in pure language-model reasoning deployments.

\begin{table}[h!]
\centering
\caption{Cross-domain conceptual mapping for Adaptive Cognitive Pacing.}
\label{tab:mapping}
\begin{tabular}{ll}
\toprule
\textbf{Paper Concept (RL)} & \textbf{Agent / LLM Analog} \\
\midrule
Action interval & Reasoning or planning frequency \\
Acting sooner & Replanning / tool invocation \\
Acting later & Waiting / batching computation \\
Hyperbolic spread & Branching of imagined futures \\
Interval-aware reward & Compute--quality trade-off \\
\bottomrule
\end{tabular}
\end{table}

\paragraph{Limitations and Future Work.}
We note three boundaries to the current empirical validation.  First, the
real-world LLM benchmark (Section~\ref{sec:real_llm}) operates at a small
scale ($N=15$ episodes) on a task where both the baseline and adaptive
controllers achieve a $100\%$ success rate.  This saturation indicates
that the benchmark does not strongly differentiate policies in terms of
capability, and therefore likely underestimates the impact of adaptive
pacing in settings where mistimed actions can lead to irreversible errors
or degraded outcomes.  While this cleanly isolates the system's
\emph{efficiency} gains (reducing token consumption by up to $10.2\%$
without degrading performance), ATCPG is designed strictly as an
orchestration layer; it does not inherently improve the underlying
reasoning capability or zero-shot accuracy of the LLM.

Second, the current pacing policy is intentionally linear.  While agents
operating in highly non-convex efficiency landscapes might theoretically
benefit from deep neural policies, introducing a heavy neural approximator
for the pacing daemon would undermine the framework's core objective of
minimising computational overhead.  The $\mathcal{O}(1)$ linear update
ensures the pacing mechanism remains strictly lightweight.

Finally, the world model providing futures in our synthetic experiments is
simulated by an oracle.  While Section~\ref{sec:real_llm} partially
addresses this by demonstrating the efficiency ordering on a live GPT-4.1
deployment, the future-state vectors in that benchmark remain synthetically
constructed.  Future implementations must integrate latent perturbation
(e.g., MC-dropout over the LLM's hidden states, as detailed below) to
generate geometric futures without autoregressive cost.  Additionally, the
Poincar\'e ball implementation clips embeddings at radius $1-\varepsilon$
to ensure numerical stability; future extensions could employ a
Lorentz/hyperboloid formulation \citep{nickel2018learning} to remove this
constraint at the cost of a higher-dimensional ambient space.

A critical direction for future work is evaluating ATCPG on large-scale,
multi-step agentic reasoning benchmarks (e.g., SWE-bench or WebArena) to
investigate whether a pacing policy $\btheta$ learned on high-volatility
search tasks transfers zero-shot to complex coding domains, and whether
dynamic pacing can actively increase absolute task success rates by
preventing premature halting or context-window exhaustion.

\paragraph{Approximating the World Model $\mathcal{W}$ for LLMs.}
In our simulation the world model generates the $n$ predicted future
trajectories $\{z^{(i)}\}_{i=1}^n$ via an oracle.  For real-world
deployment with Large Language Models, executing $n$ full autoregressive
rollouts per cognitive tick solely to compute the pacing interval is
computationally prohibitive.  We propose instantiating $\mathcal{W}$ via
\emph{Latent Perturbation} \citep{gal2016dropout}: let $h_t \in \mathbb{R}^d$
be the LLM's final hidden state before action selection.  Rather than
decoding $n$ textual futures, we apply $n$ independent stochastic dropout
masks to $h_t$, producing perturbed latent states
$\tilde{h}^{(i)}_t \sim \mathrm{Dropout}(h_t)$, which are projected directly\footnote{As in prior work, we treat dropout-induced latent perturbations as an approximate epistemic proxy rather than a calibrated posterior.}
into $\ball$ via the embedding map (Eq.~\eqref{eq:embed}).
This Bayesian proxy reduces the $\mathcal{O}(n)$ generative cost of
future simulation to $\mathcal{O}(1)$ inference, allowing the continuous
pacing daemon to operate without draining the compute budget reserved for
task execution.

\paragraph{Relation to semi-MDPs.}
The problem formulation \eqref{eq:objective} is closely related to the
average-reward objective in semi-Markov decision processes
\citep{puterman1994markov}, where sojourn times are state-dependent random
variables.  Our policy explicitly parameterises the sojourn time distribution
and optimises it online, but without requiring a full semi-MDP transition
model.

\paragraph{Hyperbolic geometry as uncertainty proxy.}
The spread signal $\curv$ acts as a calibrated uncertainty estimate
without requiring a probabilistic model.  Unlike entropy-based or ensemble
disagreement methods, it is computed geometrically and scales
$\mathcal{O}(n^2)$ in the number of futures, which is manageable for small
$n$ (we use $n=4$).  It is also differentiable w.r.t.\ the embedding,
opening a path to end-to-end learning of the embedding jointly with the
pacing policy.

\paragraph{Broader applicability.}
Autonomous agents that must allocate attention, digital assistants, robotic
planners, multi-agent systems with communication costs, face the same
decision: when to act next.  ATCPG is lightweight (one weight vector, one
phase scalar, one spread computation) and can be dropped into any agent loop
that already exposes a world model or planning module.

\clearpage
\section{Conclusion}

In continuous autonomous operation, an agent must decide not only what action to take,
but also when to act next.  While recent LLM-based and multi-agent systems research has
primarily emphasised orchestration and routing, the complementary question of
self-directed temporal pacing is less often treated as a learnable decision variable.
This work introduces ATCPG, a lightweight, learnable pacing layer that treats the
inter-tick interval as a first-class control variable, enabling agents to regulate
their own cognitive update frequency.

ATCPG combines a simple online linear pacing policy with an interval-aware shaping
reward that addresses a directional failure mode in temporal credit assignment that can
arise under na\"ive outcome-based rewards.  We show that, in pacing settings, such
rewards can be misaligned with the desired ``wait-when-needed'' behaviour, systematically
eroding the incentive to slow down under load.  By explicitly pricing the chosen
interval, ATCPG mitigates this pathology and yields substantial standalone efficiency
gains in ablation.

A central component of the framework is a predictive hyperbolic-spread signal $\curv$,
computed by embedding sampled candidate futures in the Poincar\'{e} ball.  The geometry
selectively amplifies directional disagreement near the boundary while compressing
confident and near-origin states, yielding a practical proxy for prospective
uncertainty.  Both theoretical analysis and empirical results support the use of this
geometric signal as a timing-relevant indicator of epistemic conflict.

We further introduce ATCPG-ST, a joint spatio-temporal extension that augments the
policy state with embedded position trajectories.  This spatial signal provides
additional structure for pacing decisions, yielding consistent improvements in
efficiency without modifying the learning rule.  Across ablations, we observe a
monotonic improvement chain, indicating that learning, geometric uncertainty, and
spatial context contribute complementary gains.

Across evaluated settings, adaptive pacing improves efficiency over fixed-interval
baselines at matched success rates, with the full ATCPG-ST variant achieving the
strongest gains.  These effects also transfer beyond synthetic settings: in a
small-scale LLM-agent experiment using a commercial API deployment, ATCPG reduces token
consumption while maintaining identical task success rates, indicating that adaptive
pacing can translate into practical reductions in runtime cost.

ATCPG is intentionally minimal, requiring only a single weight vector, a phase
variable, and a spread computation, and can be integrated into any agent loop
capable of producing candidate futures.  Taken together, these results support
adaptive pacing as a distinct layer of agent behaviour that complements decision-making
and planning by regulating the frequency of re-evaluation.

\clearpage
\bibliographystyle{plainnat}

\appendix

\clearpage
\section{Proof of Proposition~\ref{prop:failure} (Extended)}

We provide a full worked example.  Suppose:
$\btheta = [\theta_{\text{bias}}, \theta_f] = [60, 30]$,
$f_t = 5$ (heavy fatigue), $\Delta\wb_t = -0.3$ (overload), $\alpha = 0.1$.

\noindent
\textbf{Naive outcome reward:}
$\tilde{r}_t = 2 \Delta\wb_t = -0.6$.

\noindent
\textbf{Update:}
$\theta_f \leftarrow 30 + 0.1 \cdot (-0.6) \cdot 5 = 30 - 0.3 = 29.7$.

\noindent
Interval output: $\hat\dt = 60 + 29.7 \cdot 5 = 208.5\,$s (decreased by $1.5\,$s).
Repeated across $100$ similar ticks, the fatigue weight erodes to $\approx 0$,
and the overloaded agent ticks as fast as an idle one.

\noindent
\textbf{Interval-aware reward:}
Take $\dt_t = 60$ (base interval), $\dt_{\text{base}} = 60$:
\begin{align*}
\text{efficiency}   &= 2 \cdot (-0.3 / 60) = -0.010, \\
\text{spacing}      &= 1.5 \cdot 0.3 \cdot (60/60) = +0.450, \\
\text{spread brake} &= 0 \quad (\curv_t = 0), \\
\tilde{r}_t                 &= -0.010 + 0.450 = +0.440.
\end{align*}
\textbf{Update:}
$\theta_f \leftarrow 30 + 0.1 \cdot (+0.440) \cdot 5 = 30.22$.
The fatigue weight \emph{increases}, lengthening the interval under overload
as desired.

\clearpage
\section{Algorithm Pseudocode}

\noindent
Algorithm~1 gives the complete ATCPG control loop.

\begin{figure}[h]
\centering
\begin{minipage}{0.92\textwidth}
\hrule\smallskip
\textbf{Algorithm 1}\quad ATCPG: Adaptive Temporal Control via Predictive Geometry
\label{alg:atcpg}
\smallskip\hrule\smallskip
\begin{enumerate}[leftmargin=2em,label=\arabic*:]
\item \textbf{Require} Initial weights $\btheta$, world model $\mathcal{W}$, base interval $\dt_{\text{base}}$, initial $\Delta\wb_0 \leftarrow 0$
\item $\phi \leftarrow \text{Uniform}(0, 2\pi)$;\quad $\omega \leftarrow 0.05$
\item \textbf{for} $t = 1, 2, \ldots$ \textbf{do}
\item \quad Observe context $c_t$: priority $p_t$, fatigue $f_t$,
wellbeing $\wb_t$, performance $\rho_t$
\item \quad $\curv_t \leftarrow \textsc{ComputeSpread}(\mathcal{W}, c_t)$ \hfill (Def.~1; Eq.~\eqref{eq:curvature})
\item \quad $s_t \leftarrow [p_t, f_t, \Delta\wb_{t-1}, \rho_t, \sin\phi, \curv_t]$
\item \quad $\hat\dt_t \leftarrow \text{clip}(\btheta^\top s_t,\;
\dt_{\min}, \dt_{\max})$
\hfill (Eq.~\eqref{eq:policy})
\item \quad With prob.\ $\varepsilon_{\text{eff}}(s_t)$:
$\hat\dt_t \mathrel{\times}= \text{Uniform}(0.5, 1.5)$
\hfill (Eq.~\eqref{eq:eps})
\item \quad \textbf{Sleep} $\hat\dt_t$ seconds; execute cognitive tick
\item \quad Observe $\wb_{t+1}$;\;
$\Delta\wb_t \leftarrow \wb_{t+1} - \wb_t$
\item \quad $\tilde{r}_t \leftarrow \textsc{IntervalAwareReward}(
\Delta\wb_t, \curv_t, \hat\dt_t)$
\hfill (Eq.~\eqref{eq:reward})
\item \quad $\btheta \leftarrow \text{clip}(\btheta + \alpha \tilde{r}_t s_t,\;
[-100, 100]^7)$
\hfill (Eq.~\eqref{eq:reinforce})
\item \quad $\phi \leftarrow (\phi + \omega)\bmod 2\pi$;\;
$\omega \leftarrow \text{clip}(\omega + \alpha \tilde{r}_t \cdot 0.01,\;
[0.001,0.2])$
\item \textbf{end for}
\end{enumerate}
\smallskip\hrule
\end{minipage}
\end{figure}

\paragraph{Implementation notes.}
All experiments were implemented in Python using standard numerical and
scientific-computing libraries (e.g., NumPy) and were executed under a
controlled set of random seeds to enable consistent comparisons across
ablations.  Due to proprietary constraints, the current implementation is
not publicly released; we intend to open-source a reference implementation
pending internal legal and compliance clearance.  In the interim, the
controller, reward, and spread estimator are specified fully in closed
form (Sections~\ref{sec:policy}--\ref{sec:spatial}), and all environmental
and algorithmic hyperparameters required to replicate the reported
experimental protocol are provided in Section~\ref{sec:experiments}.

\end{document}